\documentclass{article}

\usepackage{PRIMEarxiv}

\usepackage[utf8]{inputenc} 
\usepackage[T1]{fontenc}    
\usepackage{hyperref}       
\usepackage{url}            
\usepackage{booktabs}       
\usepackage{amsfonts}       
\usepackage{nicefrac}       
\usepackage{microtype}      
\usepackage{lipsum}
\usepackage{graphicx}
\usepackage{cite}
\usepackage{amsmath,amssymb,amsfonts}
\usepackage{algorithmic}
\usepackage{graphicx}
\usepackage{textcomp}
\usepackage[normalem]{ulem}
\usepackage{soul}
\usepackage{amsthm}
\theoremstyle{plain}
\usepackage{multirow}
\usepackage{multicol}
\usepackage{pifont}
\newtheorem{lemma}{Lemma}
\graphicspath{{media/}}     

\title{Adaptive Prototype-based Interpretable Grading of Prostate Cancer}

\author{
  Riddhasree Bhattacharyya, Pallabi Dutta, Sushmita Mitra \\
  Machine Intelligence Unit, \\
  Indian Statistical Institute, \\
  Kolkata 700108 \\
  \texttt{\{riddhasreeb@gmail.com, duttapallabi2907@gmail.com, sushmita@isical.ac.in\}}\\
}

\begin{document}
\maketitle

\begin{abstract}
Prostate cancer being one of the frequently diagnosed malignancy in men, the rising demand for biopsies places a severe workload on pathologists. The grading procedure is tedious and subjective, motivating the development of automated systems. Although deep learning has made inroads in terms of performance, its limited interpretability poses challenges for widespread adoption in high-stake applications like medicine. Existing interpretability techniques for prostate cancer classifiers provide a coarse explanation but do not reveal why the highlighted regions matter. In this scenario, we propose a novel prototype-based weakly-supervised framework for an interpretable grading of prostate cancer from histopathology images. These networks can prove to be more trustworthy since their explicit reasoning procedure mirrors the workflow of a pathologist in comparing suspicious regions with clinically validated examples. The network is initially pre-trained at patch-level to learn robust prototypical features associated with each grade. In order to adapt it to a weakly-supervised setup for prostate cancer grading,  the network is fine-tuned with a new prototype-aware loss function. Finally, a new attention-based dynamic pruning mechanism is introduced to handle inter-sample heterogeneity, while selectively emphasizing  relevant prototypes for optimal performance. Extensive validation on the benchmark PANDA and SICAP datasets confirms that the framework can serve as a reliable assistive tool for pathologists in their routine diagnostic workflows.
\end{abstract}


\section{Introduction}
\label{sec:introduction}

Prostate cancer ranks as the second most frequently diagnosed malignancy in men, contributing to approximately 14.5\% of all male cancer cases globally. The standard clinical workflow for prostate cancer diagnosis and grading involves staining of biopsy-derived prostate tissue, followed by assessment using the Gleason grading system. The current system comprises of three Gleason Grades (GG) (3, 4, 5), where each grade has an inverse correlation with the degree of gland differentiation. While Gleason pattern (GP) 3 is characterized by discrete and dense glandular regions, GP 4 exhibits poorly formed, fused, cribriform, or papillary glandular structures, indicating partial loss of glandular organization. However, GP 5 represents the most poorly differentiated architecture, showing isolated tumour cells without glandular lumina, often with comedonecrosis or pseudo-rosetoid formations \cite{silva2020going}. In practice, pathologists consider the sum of both the first and second-most dominant GGs (in terms of proportion and severity) in prostate biopsies for reporting the gleason score of the entire sample.

However, this grading process is inherently tedious, demands substantial expertise, and is susceptible to both intra- and inter-observer variability. The rising demand of prostate biopsies places severe workload on the pathologists, thus hindering timely treatment. In this scenario, automated systems can alleviate the problem of subjectivity in the task; thereby, producing replicable results while increasing the throughput of pathologists. Development of automated systems using Deep Learning (DL) to analyze digitized biopsies or histopathology Whole Slide Images (WSIs) has become an interesting research area \cite{qu2022towards}. Central to this advancement is the rise of Convolutional Neural Networks (CNNs) \cite{lecun15}, which form the backbone of most modern DL systems.

However, application of DL to histology faces two prime challenges, {\it viz.} computational demand and lack of interpretability. The high resolution of WSIs makes it computationally infeasible to train the DL models end-to-end. This has led to a paradigm shift towards weakly-supervised or Multiple Instance Learning (MIL)-based algorithms \cite{qu2022towards} in histopathology, where each WSI is represented as a bag of smaller patches. The model is subsequently trained using the WSI-level label with the WSI-level representation obtained by aggregating the patch-level features or predictions \cite{gadermayr2024multiple}. Though MIL-based methods successfully address the computational bottleneck to a certain extent, their lack of sufficient interpretability poses challenges to a widespread adoption of DL in clinical settings \cite{holzinger2019causability}. The models often focus on the background or other confounding information in the images, which are considered irrelevant by the pathologists for classification \cite{barnett2021case}. For DL to be considered reliable in high-stake applications like medicine, it is important that the models are trained to focus on medically relevant information with their reasoning process being transparent. Making the architectures sufficiently interpretable is even more challenging in case of prostate WSI grading. This is because multiple cancerous patterns can co-exist in a WSI with the predictions being based on aggregated patch-level features/ predictions in MIL \cite{klauschen2024toward, keshvarikhojasteh2025beyond}.

Interpretability techniques in literature \cite{barnett2021case} can be broadly categorized into posthoc, prototype-based and attention-based interpretable models. Posthoc interpretability analysis \cite{krebs2025beyond, manz2025explainable} reflects the 
decision-making behavior of trained CNNs by highlighting the image regions or attributes having strong influence on the predictions. However, it does not uncover the exact computational logic or feature interactions that drives a decision. Furthermore, the output can be sensitive to noise, model architecture, and layer selection -- sometimes, resulting in misleading explanation \cite{adebayo2018sanity, kindermans2019reliability}.

Unlike posthoc methods that generate explanations after classification, the prototype-based inherently interpretable networks \cite{chen2019looks, barnett2021case} integrate interpretability directly into their architecture. They learn a set of class-specific prototypes during training, each representing a characteristic visual pattern of its class. The network classifies an unseen image according to its similarity to these learned prototypes, effectively following a case-based reasoning strategy.  The attention-based networks \cite{ilse2018attention, mao2025camil} are an intermediary between the posthoc and prototype-based models.  Their attention mechanism is a part of the computation rather than posthoc addition. Although attention maps explicitly indicate which regions or features a model focuses on, they fail to point prototypical cases similar to the parts to which  higher attention weights get assigned.

Researchers have experimented with different posthoc and attention-based interpretability techniques for explaining prostate cancer classification. Occlusion-based posthoc interpretability methods remove or distort meaningful image information to  analyze its impact on performance. While Ref. \cite{provenzano2022exploring} applied global occlusions such as contrast reduction, thresholding, Canny edge detection; Ref. \cite{gallo2023shedding} systematically blocked different regions of the input image to evaluate their contribution to the classification of prostate cancer. However, multiple forward passes to explain each patch from WSI makes such occlusion procedure time-consuming. Several researchers \cite{krebs2025beyond, manz2025explainable} applied single-pass posthoc techniques such as saliency mapping, guided backpropagation, and class activation mapping (CAM) \cite{selvaraju2017grad} to address this issue. In addition, simple posthoc visualization by taking patch-level cancer probabilities predicted by a MIL model and projecting them back onto the WSI to visualize tumour locations was also used in Ref. \cite{silva2021self, yang2023devil}. Use of attention-based MIL frameworks has recently become increasingly common in classifying histopathology images, including prostate cancer \cite{wang2022scl, xiang2023automatic, chaurasia2025generalised, rahman2025cemil}. Attention weights highlight diagnostically relevant tissue regions and guide decision-making in these frameworks. The forward computation of these models involve attention maps providing coarse yet intuitive visual explanation.

Existing literature does not report adapting prototype-based inherently interpretable networks for improved interpretability in prostate cancer grading systems. 
Post-hoc and attention-based approaches highlight regions that influence the model output. However, they do not reveal why those regions matter, what visual patterns the model relies on, or how those patterns relate to medically relevant examples in the training set. Absence of explicit reasoning limits trust, particularly in high-stake diagnosis where spurious correlations or background artifacts can easily mislead a model. 
Prototype-based models address this gap by embedding human-understandable visual concepts (prototypes) directly into the network decision pathway. This allows predictions to be grounded in prototypical Gleason patterns learned from real patient slides. Such models reason in a way that more closely mirrors the workflow of a  pathologist—comparing suspicious regions to previously seen and clinically validated examples—thereby, ensuring that decisions arise from medically relevant features rather than contextual bias. 

However, adapting the prototype-based network for prostate cancer grading is challenging for two main reasons. First, it is difficult to adopt it to the weakly-supervised MIL setting required for prostate cancer grading. The primary reasons are the domain shift between patch and slide-level distributions, and  the co-existence of multiple cancerous patterns in each WSI. Second, the number of prototypes being a hyperparameter, all prototypes may not be equally informative. Consequently, it becomes necessary to emphasize the more relevant prototypes and suppress the irrelevant ones for optimal performance. The static pruning mechanisms  \cite{chen2019looks, rymarczyk2022interpretable} completely eliminate the less important prototypes and/or share the common prototypes with other classes. While effective in some natural image settings, these strategies are inherently inflexible. They permanently remove prototypes and limit the ability of a model to adapt to image-specific variability. Such lack of flexibility is particularly restrictive in medical imaging, where substantial inter-slide heterogeneity is common. Moreover, prototype sharing is not well suited to prostate cancer grading, where each Gleason grade corresponds to distinct morphological characteristics.

To address these challenges, we propose a novel Attention Driven Adaptive Prototype Thresholding (ADAPT) framework for improved interpretability in grading of Prostate Cancer from WSIs. Initially, the network is pre-trained at patch-level to learn the robust prototypical features associated with each Gleason grade. Thereafter, it is innovatively fine-tuned for grading WSIs using a new loss function in an MIL setting. Finally, a novel attention-based dynamic pruning mechanism is introduced to selectively emphasize the relevant prototypes and improve performance. Schematic diagram of the workflow is provided in Fig. \ref{fig:workflow}, with main contributions summarized below. 

\begin{figure*}[!ht]
    \centering
    \includegraphics[width=0.80\textwidth, trim=0cm 0.2cm 0.2cm 0cm,
        clip]{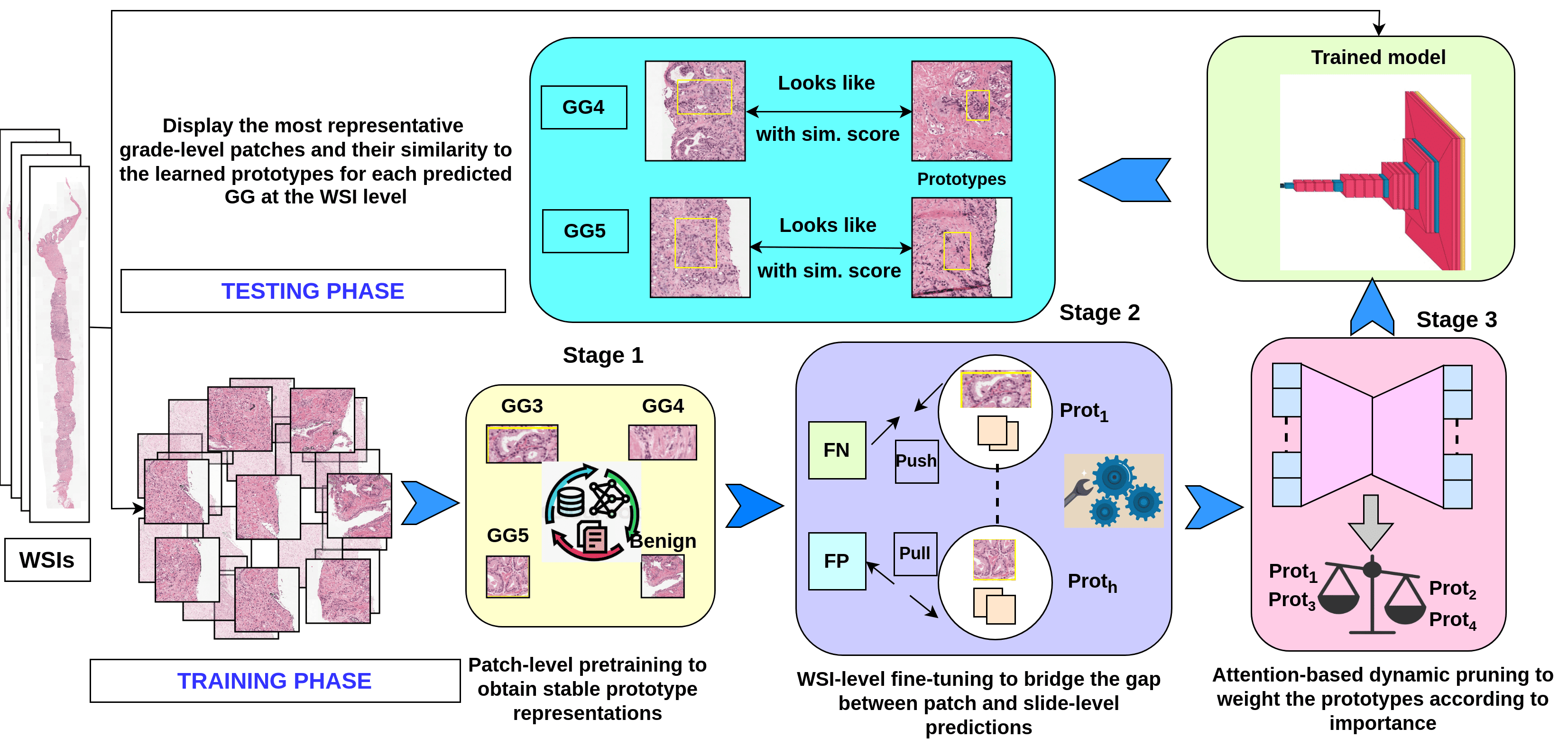}
    \caption{Schematic representation of the workflow of the ADAPT framework}
    \label{fig:workflow}
\end{figure*}

\begin{enumerate}
    \item Patch-level pretraining: Helps obtain stable prototype representations for each GG;  thereby, accelerating WSI-level convergence in a weakly supervised framework.
    
    \item Fine-tuning the patch-level prototype model at WSI-level:  A novel loss function, introduced in an MIL framework, adapts to slide-level aggregation rules and aligns domain shift between patch and slide distributions. The prototype space is adapted to capture histopathological patterns of WSI-level labels, to bridge the gap between patch decision and slide-level diagnosis.

    \item  Incorporating a new attention-based dynamic prototype pruning mechanism: Emphasizes relevant prototypes to promote class-wise exclusivity and enforce sparsity (few useful prototypes per class); thereby, reducing redundancy and improving interpretability.

    \end{enumerate}
    
Extensive validation is performed on the multi-center PANDA challenge dataset \cite{kagglepanda} and the SICAPv2 dataset \cite{silva2020going} to assess the performance and generalizability of the ADAPT framework. The results confirm the robustness of our framework and underscore its promise as a reliable assistive tool for pathologists in routine diagnostic workflows.
The rest of the article is organized as follows. Section \ref{sec:method} describes the working principle of our prototype-based interpretable framework for prostate cancer grading. Section \ref{sec:exp_dis} presents the experimental details along with the quantitative and qualitative results to demonstrate the efficacy of the proposed framework. Finally, Section \ref{sec:conc} concludes the article.

\section{Methodology}
\normalcolor
\label{sec:method}

This section discusses the problem statement, followed by a description of the three stages in the ADAPT framework.

\subsection{Problem statement}

Let $\{s_n\}_{n=1}^{N}$ denote a collection of $N$ prostate biopsy WSIs.  Each $s_n$ is expressed as a bag of $l$ patches $(s_n = \{x_{n,i}\}_{i=1}^{l})$, where $x_{n,i}$ represents the $i$-th patch extracted from the $n$-th WSI. Every WSI has two labels corresponding to the primary and secondary GG denoted by $z_{n,pg}$ and $z_{n,sg}$ respectively, where $z_{n,pg}, z_{n,sg} \in \{3,4,5\}$. Hence, the WSI grading task is formulated as a multilabel binary classification problem, where the target label for each WSI $s_n$ is encoded as $\mathbf{y}_n = \{ y_n^c \}_{c \in \{3,4,5\}}$. Here, the index $c$ denotes a GG, and 
$y_n^c = 1$ if the WSI $s_n$ exhibits grade $c$ in $z_{n,pg}$ or $z_{n,sg}$, and $y_n^c = 0$ otherwise. In addition, there is an auxiliary patch-level annotated dataset $PD =  \{(x_{n,i}, z_{n,i})\} \subset \{s_n\}_{n=1}^{N}$ where $z_{n,i}$ denotes the GG of each patch $x_{n,i}$. The aim is to develop a prototype-based DL model that not only predicts the GG w.r.t each WSI but also points out prototypical cases similar to regions of the test images, based on which the GGs are predicted. Our novel framework accomplishes this goal by using three distinct stages, as described next. 

\begin{figure*}[!ht]
    \centering
    \includegraphics[width=0.85\textwidth]{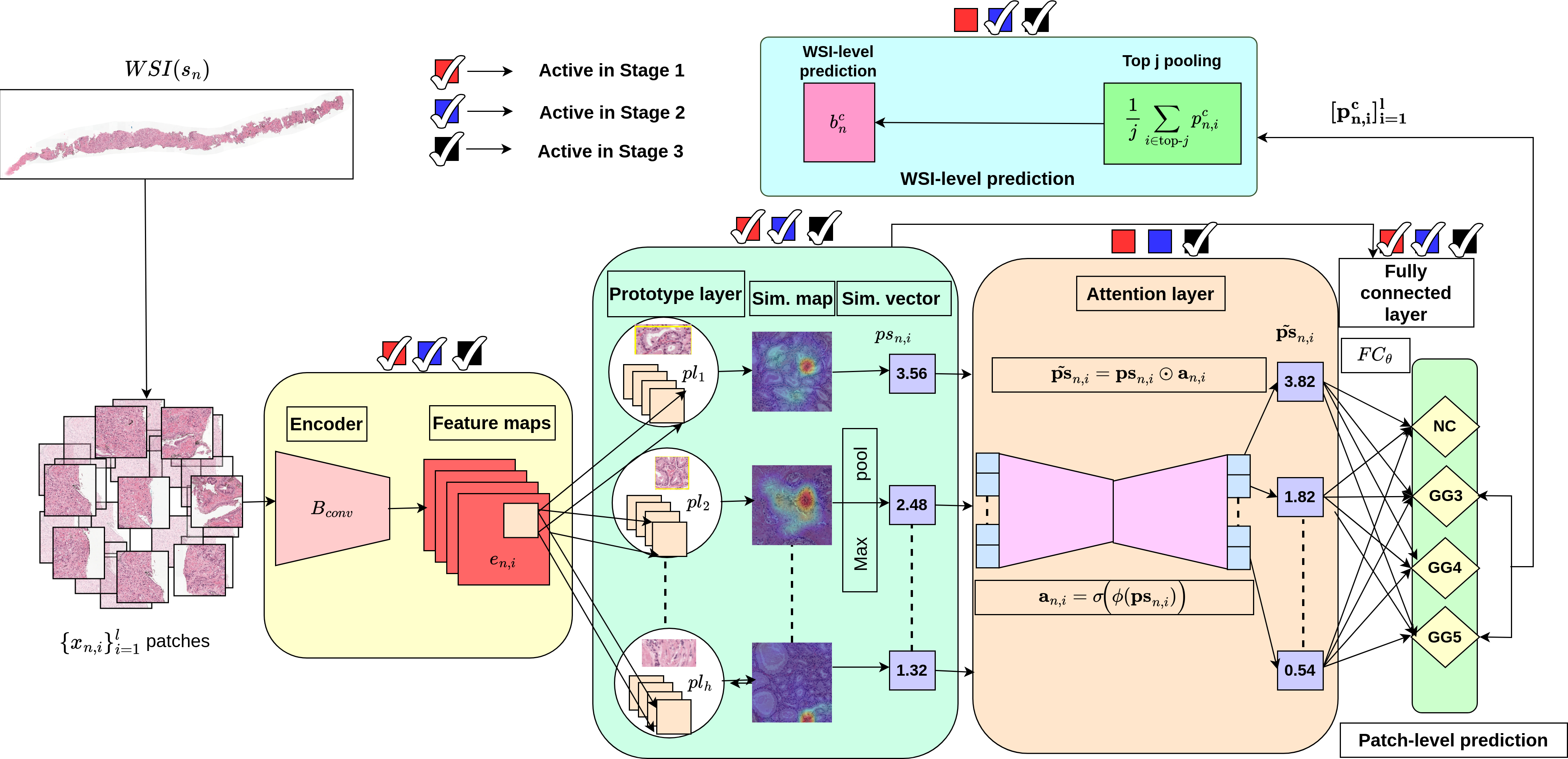}
    \caption{The architecture of the ADAPT framework. During stage 1 patch-level pretraining, the attention and WSI-level prediction modules are not needed. For stage 2 WSI-level fine-tuning, the attention module is not required. Stage 3 uses all the modules.}
    \label{fig:framework}
\end{figure*}

\subsection{Stage 1: Patch-level pretraining}

Prototype discovery solely in a weakly supervised framework like MIL is unstable and also sensitive to class imbalance. Thus, patch-level pretraining provides a good initialization for the WSI-level grading (in the MIL framework) and is also necessary to shape the prototypes into semantically meaningful Gleason-specific morphological pattern detectors.

\noindent \textbf{DL model design}: The DL model used in this stage comprises a traditional CNN backbone ($B_{conv}$) (described in Sec. \ref{sec:exp_dis}), followed by a prototype layer ($pl$) and the final fully-connected (FC) layer ($FC_{\theta}$) with four output nodes (corresponding to benign and GGs 
3-5), as depicted in Fig. \ref{fig:framework}. The prototype layer $pl = \{pl_g\}_{g=1}^h$ consists of total $h$ prototypes, where the number of prototypes allocated for each class is a hyperparameter. If the output feature map produced by $B_{conv}$ has height $H$, width $W$ and depth $D$,  the shape of each $pl_g$ is $H_1 \times W_1 \times D$ where $H_1 \le H, W_1 \le W$.  Thus, each prototype encodes a representative activation pattern for a particular class learned during training. Consequently, these prototypes are visualized as the training image patch where they are most evident.

Given an input image $x_{n,i}$, the $B_{conv}$ produces a feature map or embedding $e_{n,i}$. Consequently, the prototype layer computes the squared $L_2$ distance between the $g$th prototype $pl_{g}$ and all non-overlapping regions of $e_{n,i}$ having same spatial dimensions as $pl_{g}$. These distances are converted into similarity values to produce a similarity activation map, where higher scores indicate a stronger presence of the prototypical pattern in the image. Conceptually, if the input image $x_{n,i}$ contains a particular Gleason pattern, then the sub-regions within the corresponding feature maps $e_{n,i}$ will encode these structures and lie close to one or more prototypes representing that particular Gleason pattern. As a result, the similarity scores between those image portions and the relevant prototypes will be high. Next, the activation map generated by each prototype unit $pl_{g}$ is aggregated via global max pooling into a single similarity value. It reflects the strongest presence of that prototypical feature in the input image.

Finally, the $h$ similarity scores generated by $pl$ are passed through the fully connected layer $FC_{\theta}$, where they are linearly combined using the weight matrix $\theta$ to form the output logits. These logits are then normalized with softmax to obtain the probability of $x_{n,i}$ to have a certain GG.

\noindent \textbf{Training strategy:} The training strategy of the above DL model comprises of the following three phases, which are repeated until convergence.

\noindent \textbf{Phase 1}. First,  $B_{conv}$ and $pl_g$ are jointly optimized while keeping $FC_{\theta}$ fixed. This ensures that the discriminative features, corresponding to a particular GG, move closer to prototypes of that class in latent space and away from the others. The $FC_{\theta}$ layer is kept fixed by assigning a positive weight (1) between the logit of a particular class and its corresponding prototypes, and a negative weight (-0.5) between this class logit and the remaining non-class prototypes. The network is encouraged to learn a meaningful latent space by initially constraining $FC_{\theta}$ in this manner. The intuition is that if a latent patch from an image having a particular GG lies too close to a prototype associated with a different GG, the predicted probability for that grade would decrease. This would lead to a higher cross-entropy loss during training. The total training loss w.r.t. each patch $x_{n,i}$ is defined as
\begin{equation}
\begin{aligned}
L_{\text{patch}} =\; & CE\!\left(FC_{\theta}\!\left[pl\!\left\{B_{\text{conv}}(x_{n,i})\right\}\right], z_{n,i}\right) + \lambda_{\text{clst}}\, clst\\
& + \lambda_{\text{sep}}\, sep,
\end{aligned}
\label{eqn:loss_patch}
\end{equation}
\noindent where CE denotes Cross Entropy Loss between the predicted GG and the corresponding patch label, the cluster ($clst$) and separation ($sep$) losses are defined by eqn. (\ref{clsep}), with hyperparameters $\lambda_{clst}$, $\lambda_{sep}$ adjusting their respective contributions. 
\begin{equation}
\label{clsep}
\begin{aligned}
clst &= \min_{g : \mathbf{pl}_g \in \mathbf{pl}_{z_{n,i}}}
               \; \min_{\, f \in \text{patches}(e_{n,i})}
               \; \| f - \mathbf{pl}_g \|_2^{2}, \\
sep  &= - \min_{g : \mathbf{pl}_g \notin \mathbf{pl}_{z_{n,i}}}
               \; \min_{\, f \in \text{patches}(e_{n,i})}
               \; \| f - \mathbf{pl}_g \|_2^{2}.
\end{aligned}
\end{equation}

\noindent \textbf{Phase 2}. Each prototype $pl_g$ is pushed onto the closest latent patch belonging to the same class as itself. This aligns every prototype with its nearest real patch, allowing it to be directly interpreted as a representative example of the training data.

\noindent \textbf{Phase 3}. The aim is to adjust the parameters of only the $FC_{\theta}$ layer, while keeping the $B_{conv}$ and $pl$ layers fixed. This further enhances the performance without affecting the learned latent space or prototypes.

\subsection{Stage 2: WSI-level fine-tuning}

The patch-level pretrained model (from Stage 1) is fine-tuned for WSI-level grading in an MIL framework. This ensures learning of robust slide-level aggregation rules while fixing domain shift between the patch-level and slide-level distributions. This helps correct interpretation of the relative importance of the most confident patches by aligning prototype activations with WSI-level evidence, while avoiding any overemphasis on the ambiguous/noisy patches.

The DL model in our MIL setup consists of the Stage 1 network (comprising of the backbone, prototype, and fully-connected layers) followed by a WSI-level prediction module. It can be trained end-to-end using WSI labels, as depicted in Fig. \ref{fig:framework}. First, the stage 1 patch-level classifier (i.e., the pretrained patch-level prototype network) estimates, for each patch $x_{n,i}$ of slide $s_n$, the probability $p_{n,i}^c$ of belonging to one of the four classes or Gleason grades $c$. Here, $c=0$ represents the benign tissue, while $c=3,4,5$ correspond to GG3, GG4, and GG5, respectively. Then, the WSI-level prediction module aggregates these patch-level predictions into WSI-level scores. For each Gleason grade $c \in \{3,4,5\}$, the set of probabilities $\{p_{n,i}^c\}_{i=1}^{l}$ is converted into a single WSI-level probability $b_n^c$ by averaging the top $j$ most confident patch predictions.
\begin{equation}
\label{avgtopk}
    [b_n^c] = \frac{1}{j} \sum_{i \in \text{top-}j} p_{n,i}^c.
\end{equation}
Here, the operator top-$j$ selects the $j$ highest values among $p_{n,i}^c$ of $l$ patches in each WSI. The top-$j$ averaging provides a robust compromise between two extremes. While averaging over all patches can dilute the contribution of small but clinically relevant tumour regions, relying solely on the maximum patch score can be highly sensitive to outliers or spurious high-confidence predictions. Averaging the top $j \in [3,7]$ probabilities is found to provide resilience to single-patch noise; thereby, leading to a more stable and reliable WSI-level probability estimate. 

This WSI-level DL model is now fine-tuned using a novel loss function. Alongside the standard WSI-level multilabel Binary Cross-Entropy (BCE) loss between predicted bag-based probability $\mathbf{[b_n^c]}$ at WSI-level and the one-hot encoded target output $\mathbf{[y_n^c]}$, we introduce two complementary prototype-aware regularizers to correct any systematic failure modes arising during aggregation. This includes a Positive alignment loss that helps the model recover missed evidence of true classes, and a Negative repulsion loss that suppresses misleading evidence responsible for false positives, as illustrated in Fig. \ref{fig:wsiloss}.

\begin{figure}[!ht]
    \centering
    \includegraphics[width=0.45\textwidth]{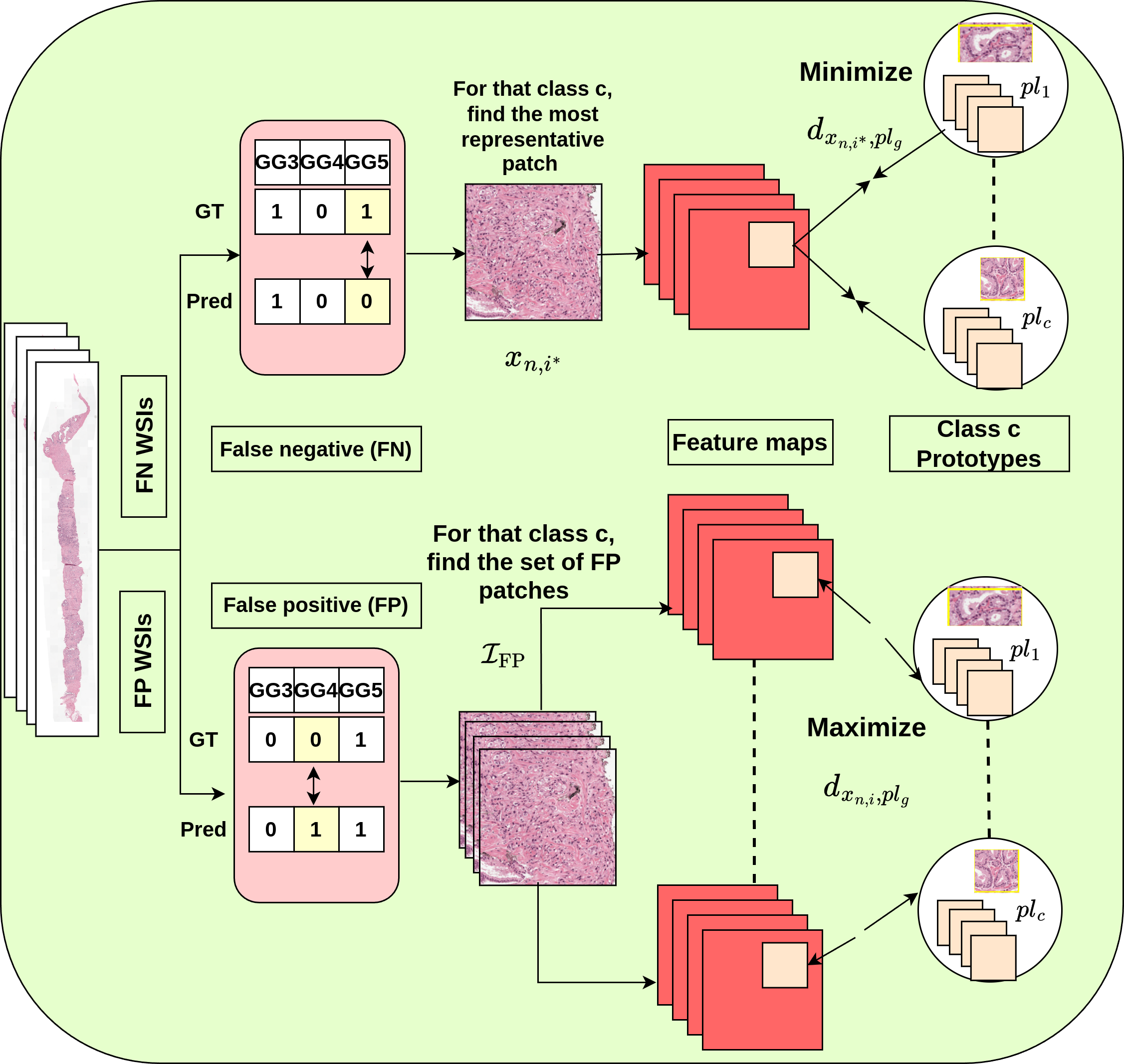}
    \caption{Training strategy for WSI-level fine-tuning}
    \label{fig:wsiloss}
\end{figure}

\noindent \textbf{Positive alignment loss}: When the model fails to detect a true GG at the slide level, it typically means that patches containing relevant morphological patterns are not sufficiently aligned with the prototypes assigned to that grade. To address this factor, the Positive alignment loss encourages the most representative patches in the slide to move closer to the correct class prototype(s).

Let $d_{x_{n,i},pl_g}$ denote the distance between features of patch $x_{n,i}$ and prototype $pl_g \in pl_c$, where $pl_c$ represents the set of prototypes associated with class $c$. For a given slide $s_n$, let $x_{n,i^*}$ denote the patch with the highest predicted probability for class $c$ when the aggregated slide-level probability still remains below $0.5$ (corresponding to a false-negative detection). Positive alignment loss is formulated as
\begin{equation}
\mathcal{L}_{\mathrm{align}} = \frac{1}{N_{\mathrm{FN}}} \sum_{\text{FN slides}}\min_{pl_g \in pl_c} d_{x_{n,i^*},pl_g},
\end{equation}
where $N_{\mathrm{FN}}$ denotes the number of false-negative WSIs. This term encourages the model to discover appropriate class-specific regions which got overlooked during aggregation.

\noindent \textbf{Negative repulsion loss}: Conversely, false positives can occur when certain patches spuriously activate prototypes of an incorrect class. To prevent such misleading activations from influencing the WSI prediction, the negative repulsion loss pushes such patches (with probability $>$ 0.5) away from the prototypes they incorrectly resemble. Let $\mathcal{I}_{\mathrm{FP}}$ denote the set of indices of false-positive patches in $s_n$, defined as those patches whose predicted probability for an incorrect class exceeds 0.5.
The loss is defined as
\begin{equation}
\mathcal{L}_{\mathrm{repel}} = - \frac{1}{N_{\mathrm{FP}}} \sum_{\text{FP slides}} \left(
\frac{1}{|\mathcal{I}_{\mathrm{FP}}|} \sum_{i \in \mathcal{I}_{\mathrm{FP}}} \min_{pl_g 
\in pl_c} d_{x_{n,i},pl_g} \right),
\end{equation}
where $N_{\mathrm{FP}}$ denotes the number of false-positive WSIs. This explicitly increases the distance between unreliable patches and class prototypes; thereby, reducing the likelihood of model-driven false alarms at the slide level.

The WSI-level loss function is expressed as 
\begin{equation}
\mathcal{L}_{\text{WSI}} = \mathcal{L}_{\text{BCE}} + \lambda_{\text{align}} \, \mathcal{L}_{\text{align}} + \lambda_{\text{repel}} \, \mathcal{L}_{\text{repel}},
\label{eq:wsi}
\end{equation}
where $\lambda_{\text{align}}$ and $\lambda_{\text{repel}}$ regulate the influence of the alignment and repulsion terms respectively. Together, these prototype-aware regularizers enhance the ability of the model to aggregate patch-level evidence in a stable, interpretable, and clinically grounded manner.  

\subsection{Stage 3: Attention-based dynamic prototype pruning}

After WSI-level fine-tuning in Stage 2, all prototypes are still considered as equally informative. Given that the number of prototypes is a hyperparameter, many learned prototypes can be redundant or capture non-discriminative patterns (like background or cross-class features). Thus, to enable the model to dynamically emphasize diagnostically relevant prototypes while suppressing the less informative ones, a learnable attention layer is introduced. It assigns relevance scores to the pretrained set of prototypes. A dynamic attention layer, as shown in Fig. \ref{fig:framework}, lies between the prototype layer $pl$ and the final fully-connected layer $FC_{\theta}$ to learn prototype-specific relevance weights corresponding to each patch. 

Given the prototype similarity vector $\mathbf{ps}_{n,i} \in \mathbb{R}^{K}$ for patch $x_{n,i}$ (from minimum patch–prototype distances), where $K$ denotes the total number of prototypes across all GGs, the attention module predicts per-prototype weights 
\begin{equation}
\mathbf{a}_{n,i} = \sigma\!\big\{\phi(\mathbf{ps}_{n,i})\big\} \in [0,1]^K,
\end{equation}
where $\phi(\cdot)$ denotes a lightweight two-layer MLP and $\sigma(\cdot)$ is the sigmoid activation. The weighted prototype representation is then computed as
\begin{equation}
\tilde{\mathbf{ps}}_{n,i} = \mathbf{ps}_{n,i} \odot \mathbf{a}_{n,i},
\end{equation}
and forwarded to the FC layer for classification.

We propose a novel classwise attention-discriminative loss that encourages (i) classwise exclusivity: a prototype should be active primarily for slides of its associated class and suppressed for negatives, and (ii) sparsity: only the few most relevant prototypes contribute to each class decision. This trains only the attention and FC layers, while keeping the prototypes fixed. For each class $c$, the top-$j$ most confident patches are first selected for each WSI $s_n$  based on their class probabilities
$T_{n,c}=\text{Top-}j\big(\{p_{n,i}^c\}_{i=1}^{l}\big).$
Their attention values over prototypes assigned to class $c$ (corresponding to $pl_c$) are averaged to obtain a WSI-level attention vector
\begin{equation}
\label{eq:att_vect}
\bar{\mathbf{a}}_{n,c} = \frac{1}{|T_{n,c}|} \sum_{i \in T_{n,c}} \mathbf{a}_{n,i}^{(c)},
\end{equation}
where $\mathbf{a}_{n,i}^{(c)}$ denotes the sub-vector of $\mathbf{a}_{n,i}$ corresponding to prototypes in $pl_c$.
For each grade $c$, the mean attention responses are computed, separately over the sets of positive WSIs ($\mathcal{S}_c^{+}$) and negative WSIs ($\mathcal{S}_c^{-}$) (based on whether grade $c$ is present or absent), as
$\boldsymbol{\mu}_c^{+}=\frac{1}{|\mathcal{S}_c^{+}|}\sum_{n\in\mathcal{S}_c^{+}}\bar{\mathbf{a}}_{n,c}$ and
$\boldsymbol{\mu}_c^{-}=\frac{1}{|\mathcal{S}_c^{-}|}\sum_{n\in\mathcal{S}_c^{-}}\bar{\mathbf{a}}_{n,c}$.
Prototypes exhibiting strong responses on both positive and negative samples are penalized through an overlap term
$\|\boldsymbol{\mu}_c^{+}\odot\boldsymbol{\mu}_c^{-}\|_{1}$, where $\odot$ denotes the Hadamard or element-wise product and $\|.\|_{1}$ refers to the L1 norm. Prototypes highly active on negative WSIs are suppressed using a negative-activity penalty
$\|\boldsymbol{\mu}_c^{-}\|_{1}$.
The classwise discriminative loss is then defined as
\begin{equation}
\mathcal{L}_c = w_c \big( \alpha \|\boldsymbol{\mu}_c^{+}\odot\boldsymbol{\mu}_c^{-}\|_{1}
+ \beta \|\boldsymbol{\mu}_c^{-}\|_{1} \big)
\end{equation}
with $w_c=|\mathcal{S}_c^{+}|/B$ denoting the fraction of positive WSIs for class $c$ in the batch and $\alpha,\beta$ controlling the relative strengths of the two penalties. Here, $w_c$ scales classwise loss according to the prevalence of class $c$ in the batch, ensuring stable and imbalance-aware optimization in multilabel setting.

\begin{lemma}
Let $\epsilon \ge 0$ be the convergence value of the classwise loss function such that $\mathcal{L}_c \le \epsilon$. The mean attention responses for positive and negative WSIs have disjoint support quantified by the bound on their inner product $\langle \boldsymbol{\mu}_c^+, \boldsymbol{\mu}_c^- \rangle \le \frac{\epsilon}{w_c \alpha}$. For $\epsilon \to 0$, the supports of the class representations become strictly disjoint $\text{supp}(\boldsymbol{\mu}^+) \cap \text{supp}(\boldsymbol{\mu}_c^-) = \emptyset$.
\end{lemma}
\begin{proof}
The prototype attention weights are derived via strictly non-negative activation functions. Thus, $\boldsymbol{\mu}_c^{+}$ and $\boldsymbol{\mu}_c^{-}$ lie in the non-negative orthant
\begin{equation*}
    \boldsymbol{\mu}_{c,k}^{+} \ge 0, \quad \boldsymbol{\mu}_{c,k}^{-} \ge 0 \quad \forall k \in \{1, \dots, pl_c\}.
\end{equation*}
The loss function $L_{c}$ acts as an upper bound on its constituent terms. Therefore, for the overlap penalty term
\begin{equation} \label{eq:ineq}
    w_c \alpha ||\boldsymbol{\mu}_{c,k}^{+} \odot \boldsymbol{\mu}_{c,k}^{-}||_1 \le \mathcal{L}_c \le \epsilon.
\end{equation}
For non-negative vectors, the $L_1$ norm of the Hadamard product is mathematically identical to the Euclidean inner product. We expand the term
\begin{equation*}
    ||\boldsymbol{\mu}_{c,k}^{+} \odot \boldsymbol{\mu}_{c,k}^{-}||_1 = \sum_{k=1}^{pl_c} |\boldsymbol{\mu}_{c,k}^+ \cdot \boldsymbol{\mu}_{c,k}^-| = \sum_{k=1}^{pl_c} \boldsymbol{\mu}_{c,k}^+ \boldsymbol{\mu}_{c,k}^- = \langle \boldsymbol{\mu}_c^+, \boldsymbol{\mu}_c^- \rangle.
\end{equation*}
Substituting this into eqn. (\ref{eq:ineq}) yields the bound
\begin{equation} \label{eq:bound}
    \langle \boldsymbol{\mu}_c^+, \boldsymbol{\mu}_c^- \rangle \le \frac{\epsilon}{w_c \alpha}.
\end{equation}
Let $S^+ = \{k : \boldsymbol{\mu}_{c,k}^+ > 0\}$ and $S^- = \{k : \boldsymbol{\mu}_{c,k}^- > 0\}$ be the support sets (active prototypes) for the positive and negative classes, respectively. Rewriting eqn. (\ref{eq:bound}) as 
\begin{equation*}
    \therefore \langle \boldsymbol{\mu}_c^+, \boldsymbol{\mu}_c^- \rangle = \sum_{k \in S^+ \cap S^-} \boldsymbol{\mu}_{c,k}^+ \boldsymbol{\mu}_{c,k}^- + \sum_{k \notin S^+ \cap S^-} 0 \le \frac{\epsilon}{w_c \alpha}.
\end{equation*}
In ideal loss convergence scenario ($\epsilon \rightarrow 0$), we have $\sum_{k \in S^+ \cap S^-} \boldsymbol{\mu}_{c,k}^+ \boldsymbol{\mu}_{c,k}^- = 0$. Since $\boldsymbol{\mu}_{c,k}^+ \ge 0$ and $\boldsymbol{\mu}_{c,k}^- \ge 0$, the sum can be zero if and only if every individual term is zero. This necessitates that, for every prototype $k$,
\begin{equation*}
    \boldsymbol{\mu}_{c,k}^+ = 0 \quad \lor \quad \boldsymbol{\mu}_{c,k}^- = 0
    \implies \text{supp}(\boldsymbol{\mu}_c^+) \cap \text{supp}(\boldsymbol{\mu}_c^-) = \emptyset.
\end{equation*}
Thus, minimizing $\mathcal{L}_c$ guarantees no single prototype can be active for both positive and negative WSIs -- guaranteeing feature disjointedness.
\hfill $\blacksquare$
\end{proof}

\begin{lemma}
Let $\mathcal{L}_c^{neg} = w_c \beta ||\boldsymbol{\mu}_c^-||_1$ be the negative activity penalty term. For any prototype $k$ with a non-zero activation on negative samples ($\boldsymbol{\mu}_{c,k}^- > 0$), the magnitude of the sub-gradient satisfies $\left| \frac{\partial \mathcal{L}_c^{neg}}{\partial \boldsymbol{\mu}_{c,k}^-} \right| \ge w_c \beta$. This lower bound is independent of the activation magnitude $\boldsymbol{\mu}_{c,k}^-$, ensuring persistent suppression of weak background noise.
\end{lemma}
\begin{proof} 
The negative activity penalty is defined as $\mathcal{L}_c^{neg} = w_c \beta \sum \boldsymbol{\mu}_{c,k}^-$ (since $\boldsymbol{\mu}_{c,k}^- \ge 0$). For any active prototype $k$, the gradient is defined as
\begin{equation}
    \frac{\partial\mathcal{L}_c^{neg}}{\partial \boldsymbol{\mu}_{c,k}^-} = w_c \beta.1 \implies \left| \frac{\partial \mathcal{L}_c^{neg}}{\partial \boldsymbol{\mu}_{c,k}^-} \right| = w_c \beta.
\end{equation}
The derived magnitude satisfies the lower bound condition ($\ge w_c \beta$). This expression is constant and independent of the value of $\boldsymbol{\mu}_{c,k}^-$, guaranteeing a persistent suppression even for infinitesimally small activations. Therefore, weak background noise is effectively driven to zero.
\hfill $\blacksquare$
\end{proof}

The final attention regularization term is obtained by averaging across the $\mathcal{C}$ number of GG classes as 
\begin{equation}
\mathcal{L}_{\text{attn}} = \frac{1}{|\mathcal{C}|} \sum_{c\in\mathcal{C}} \mathcal{L}_c,
\label{eq:attn}
\end{equation}
and adding to the WSI-level BCE loss to train the attention and FC layers. This formulation promotes selective, class-consistent prototype usage, and improves both interpretability and robustness of WSI-level predictions. 

After training the dynamic attention module, we quantify the relevance of each prototype to individual Gleason grades by computing a classwise prototype importance score from the learned attention weights. For each class $c$ prototype, the WSI-level attention vectors [eqn. (\ref{eq:att_vect})] are averaged across all WSIs positive for class $c$. This yields a mean attention score, per class $c$ prototype, reflecting its global contribution to that grade. The resulting scores are used to rank and visualize prototypes based on their  importance for each class; thereby, providing a quantitative measure of prototype relevance and enabling improved interpretability of class-specific model decisions. Loss functions for each of the stages,  $\mathcal{L}_{\text{patch}}$ 
[eqn. (\ref{eqn:loss_patch})], $\mathcal{L}_{\text{WSI}}$ [eqn. (\ref{eq:wsi})], $\mathcal{L}_{\text{attn}}$ [eqn. (\ref{eq:attn})], and their corresponding component weights, are summarized in Table \ref{tab:imp_det}.

\section{Experimental Results and Discussion}
\label{sec:exp_dis}

Here we describe the dataset characteristics, implementation details, and evaluation metrics, followed by quantitative and qualitative results with detailed analysis and discussion.

\noindent \textbf{Dataset description:} The ADAPT framework was trained using the multicenter prostate WSI dataset from the Kaggle Prostate cANcer graDe Assessment (PANDA) Challenge \cite{kagglepanda} and tested on both the PANDA and the SICAPv2 \cite{silva2020going} datasets. Table \ref{tab:dataset_dist} illustrates the dataset composition used at patch-level and WSI-level. The patch-level dataset includes tissue patches from all GGs in the PANDA dataset, while the WSI-level PANDA and SICAP datasets are summarized by Gleason Score (GS) categories (6–10) for prostate cancer grading. The GS for each WSI is the sum of primary and secondary GGs. Training, validation, and test splits were performed in a ratio of 70:10:20 for the PANDA dataset, while the 
out-of-distribution SICAP dataset was used entirely for testing.

\begin{table}[!ht]
\centering
\caption{Dataset distribution at (a) patch-level and (b) WSI-level}
\label{tab:dataset_dist}

\begin{minipage}[t]{0.45\linewidth}
\centering
\textbf{(A) Patch-level dataset}

\vspace{2mm}
\begin{tabular}{|l|l|}
\hline
\textbf{Grade} & \textbf{\# Patches} \\
\hline
Benign & 4000 \\ \hline
GG3 & 4247 \\ \hline
GG4 & 4186 \\ \hline
GG5 & 4039 \\ \hline
\hline
\end{tabular}
\end{minipage}
\hfill
\begin{minipage}[t]{0.45\linewidth}
\centering
\textbf{(B) WSI-level dataset}

\vspace{2mm}
\begin{tabular}{|r|c|c|}
\hline
\textbf{GS} & \multicolumn{1}{c|}{\textbf{\begin{tabular}[c]{@{}c@{}}WSIs\\ (PANDA)\end{tabular}}} & \multicolumn{1}{c|}{\textbf{\begin{tabular}[c]{@{}c@{}}WSIs\\ (SICAP)\end{tabular}}} \\ \hline
6           & 1192                                                                                 & 14                                                                                   \\ \hline
7           & 1826                                                                                 & 45                                                                                   \\ \hline
8           & 1173                                                                                 & 18                                                                                   \\ \hline
9           & 1061                                                                                 & 35                                                                                   \\ \hline
10          & 109                                                                                  & 7                                                                                    \\ \hline
\end{tabular}
\end{minipage}

\end{table}

\noindent \textbf{Implementation details: }The ADAPT framework was implemented in PyTorch (version 1.13.1) with Python 3.9 and trained using a dedicated NVIDIA RTX A6000 GPU (48 GB memory). EfficientNet-B0 \cite{tan2019efficientnet} was adopted as the backbone network ($B_{conv}$) due to its lightweight architecture and computational efficiency. Patches of size $256\times256$ were extracted from each WSI and the background-dominated patches were eliminated to ensure sufficient tissue content per WSI. The Adam optimizer was used in all the stages, with the rest of the implementation details summarized in Table \ref{tab:imp_det}. The values of $\lambda_{\text{clst}}$ and $\lambda_{\text{sep}}$ were chosen following Ref. \cite{chen2019looks}. 

The terms $\lambda_{\text{align}}$ and $\lambda_{\text{repel}}$ in Stage 2 act as auxiliary regularizers that are activated only for misclassified slides. These were therefore kept one to two orders of magnitude smaller than the primary BCE loss to avoid dominating the optimization. The alignment loss was assigned a slightly higher weight, as recovering missed class evidence in FN slides is more critical than suppressing the occasional FP patches (which may arise from local noise or morphological overlap). The overlap penalty weight $\alpha$ was chosen from a moderate range, to strongly penalize prototypes that receive high attention from both positive and negative WSIs; thereby, encouraging class-exclusive prototype usage. In contrast, the negative activity penalty weight $\beta$ was selected from a lower range to softly suppress residual prototype activations on negative WSIs without excessively diminishing weak but potentially informative responses. Maintaining $\alpha>\beta$ ensured that the prevention of cross-class prototype overlap remained the dominant objective, while negative activity suppression acted as a secondary regularizer for improved robustness. 

\begin{table}[!ht]
\caption{Implementation details of the ADAPT framework}
\centering
\begin{tabular}{|l|l|l|c|l|}
\hline
\textbf{Stage} & \textbf{Epochs} & \textbf{Batch size} & \textbf{Loss weights}                                                        & \textbf{Learning rate (LR)}                                                        \\ \hline
1              & 30              & 64                  & \begin{tabular}[c]{@{}c@{}}$\mathcal{L}_{\text{patch}}$:\\ ($\lambda_{\text{clst}}\:0.8, \lambda_{\text{sep}}\:0.08$ \cite{chen2019looks})\end{tabular}    & \begin{tabular}[c]{@{}l@{}}Cosine annealing\\ (Initial LR: 3e-4)\end{tabular}      \\ \hline
2              & 25              & 8                   & \begin{tabular}[c]{@{}c@{}}$\mathcal{L}_{\text{WSI}}$:\\ ($\lambda_{\text{align}}: [0.03,0.1], \lambda_{\text{repel}}: [0.01,0.03]$)\end{tabular} & \begin{tabular}[c]{@{}l@{}}Cosine annealing\\ (Initial LR: 1e-5-1e-3)\end{tabular} \\ \hline
3              & 20              & 8                   & \begin{tabular}[c]{@{}c@{}}
$\mathcal{L}_{\text{attn}}$:\\ ($\alpha$: [0.5,2], $\beta$: [0.1,1])\end{tabular} & \begin{tabular}[c]{@{}l@{}}Cosine annealing\\ (Initial LR: 1e-3)\end{tabular}   \\ \hline  
\end{tabular}
\label{tab:imp_det}
\end{table}

\noindent \textbf{Performance metrics:} The quantitative performance of the framework was assessed using the class-wise F1 score \cite{sokolova2009systematic} and the Hamming loss \cite{zhang2013review}. Together, these metrics provide complementary insight with F1 highlighting class-specific grading reliability and Hamming loss reflecting overall labelling consistency under the multilabel WSI grading setting. Qualitative interpretability was examined through visualization w.r.t. the corresponding Gleason-annotated masks, allowing verification of prototype localization and attention-weighted evidence regions used.

\subsection{Quantitative analysis}

Ablation study was conducted on the PANDA dataset to evaluate the contributions of the three stages of the patch-level pretraining (Stage 1), WSI-level fine-tuning (Stage 2), and attention-based dynamic prototype pruning (Stage 3) of the ADAPT framework, across different prototype counts ranging between 3–6. Performance was assessed using macro F1 scores and Hamming loss at the WSI level, as depicted in Fig. \ref{fig:ablation}. Consistent improvements were observed from Stage 1 through Stage 3, in all prototype configurations, demonstrating the importance of WSI-level optimization and attention refinement. In particular, WSI-level fine-tuning yielded a substantial gain over patch-level pretraining alone (e.g., F1 increased from approximately 0.62–0.65 to 0.77–0.81, with a concurrent reduction in Hamming loss from 0.38–0.44 to 0.19–0.24). This highlights the need to learn slide-level aggregation rules beyond localized 
patch-level supervision. The introduction of the attention module provided incremental but consistent benefits, by suppressing less informative prototypes and emphasizing class-relevant evidence. This led to further F1 improvement, with lowest overall Hamming loss across configurations.

\begin{figure}[!ht] 
    \centering
    \includegraphics[width=1\linewidth]{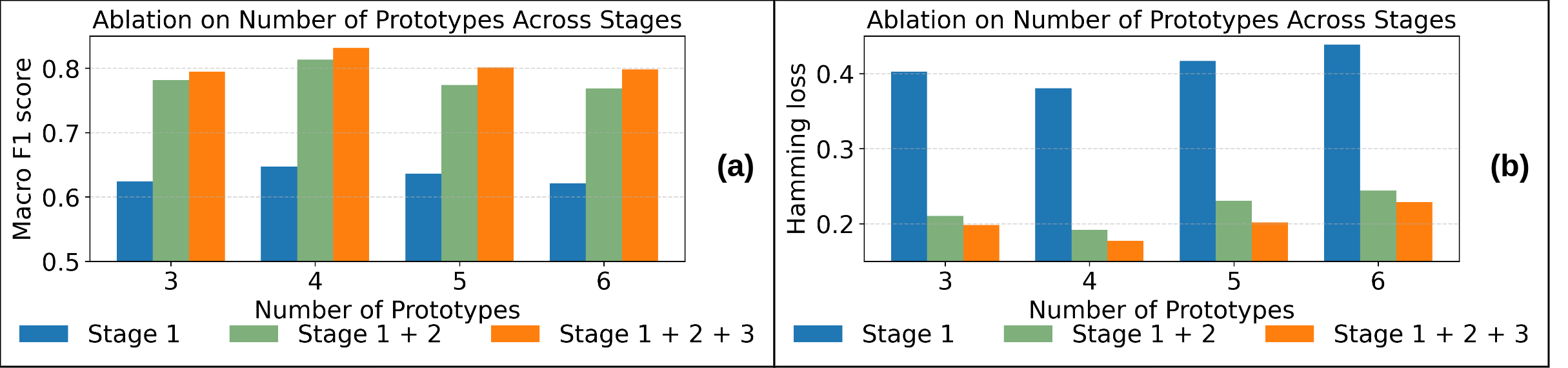}
    \caption{Ablation study evaluating the impact of the training stages and prototype counts on WSI-level Gleason grading performance over the PANDA dataset w.r.t. (a) Macro F1 score, and corresponding (b) Hamming loss}
    \label{fig:ablation}
\end{figure}

Analysis of prototype counts indicates that the model achieved optimal performance with four prototypes per class. Using only three prototypes appeared to be too restrictive, limiting the ability of the model to capture the morphological heterogeneity of prostate cancer patterns. Conversely, increasing the number of prototypes to five or six introduced additional model complexity with degraded performance, due to the emergence of redundant and non-discriminative prototypes having little complementary information. Fig. \ref{fig:prot_frac} shows that an increase in the number of prototypes per class led to a rise in the fraction of low-attention prototypes (weight$<$0.5). This indicates increasing redundancy and reduced marginal utility with additional prototypes. In particular, the attention module in Stage 3 yielded comparatively larger performance gains for higher prototype counts by suppressing less relevant prototypes. This confirmed its effectiveness in maintaining discriminative prototype usage, in spite of the overall results remaining inferior compared to the more compact four-prototype configuration. Four prototypes per class were observed to provide an effective trade-off between expressive capacity and compactness. Hence, this configuration was selected for subsequent experiments.

\begin{figure}[!ht]
    \centering
    \includegraphics[width=0.6 \linewidth, trim=0.2cm 0.2cm 0.2cm 0.2cm, clip]{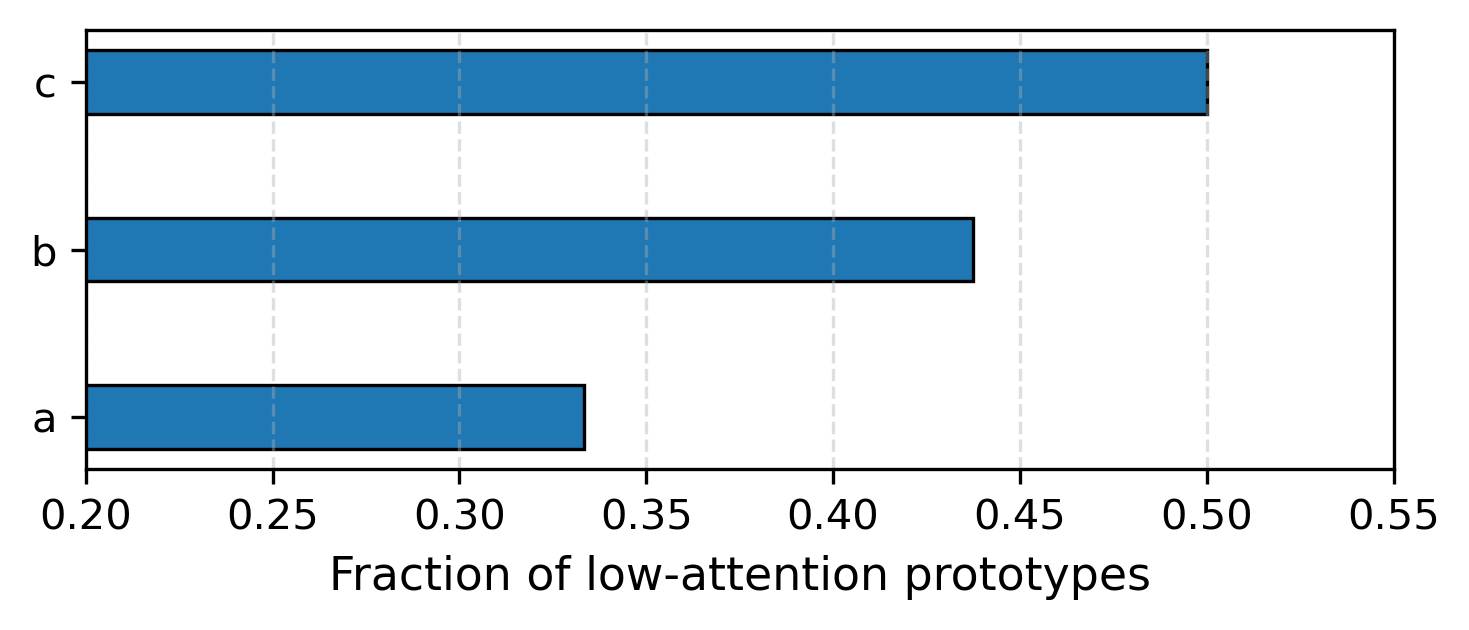}
    \caption{Fraction of low-attention prototypes, for configurations having prototypes per class as (a) 4, (b) 5, and (c) 6, over the PANDA dataset}
    \label{fig:prot_frac}
\end{figure}

Ablation on Stage 3 modules, over the PANDA dataset in Table \ref{att_ablation}, highlights that the attention block and $\mathcal{L}_{\text{attn}}$ loss play complementary roles. While the attention block enabled adaptive prototype weighting at the slide level, the BCE alone could not explicitly discourage prototypes that activated across both relevant and irrelevant classes. The classwise discriminative loss added this missing supervision signal, by penalizing cross-class and negative-slide activations; thereby, leading to improved performance.

\vspace{-4pt}
\begin{table}[!ht]
\centering
\caption{\centering Ablation study of the Stage 3 modules on the PANDA dataset}
\begin{tabular}{|lll|ll|}
\hline
\multicolumn{3}{|c|}{\textbf{Modules}}                                                                                & \multicolumn{2}{c|}{\textbf{Metrics}}                                               \\ \hline
\multicolumn{1}{|c|}{\textbf{Attention block}} & \multicolumn{1}{c|}{\textbf{BCE}} & \multicolumn{1}{c|}{\textbf{$\mathcal{L}_{\text{attn}}$}} & \multicolumn{1}{c|}{\textbf{F1 score}} & \multicolumn{1}{c|}{\textbf{Hamming loss}} \\ \hline
\multicolumn{1}{|l|}{}                         & \multicolumn{1}{l|}{}             &                                  & \multicolumn{1}{l|}{0.8130}            & 0.1916                                     \\ \hline
\multicolumn{1}{|l|}{\textbf{\ding{51}}}                     & \multicolumn{1}{l|}{\textbf{\ding{51}}}         &                                  & \multicolumn{1}{l|}{0.8255}            & 0.1800                                     \\ \hline
\multicolumn{1}{|l|}{\textbf{\ding{51}}}                     & \multicolumn{1}{l|}{\textbf{\ding{51}}}         & \textbf{\ding{51}}                             & \multicolumn{1}{l|}{0.8312}            & 0.1769                                     \\ \hline
\end{tabular}
\label{att_ablation}
\end{table}
\vspace{-4pt}

The patch-level prototype cross-activation matrix, on the PANDA dataset, was analyzed in Fig. \ref{fig:prot_matrix} to verify whether the framework could induce class-selective prototype behavior. The matrix exhibits clear diagonal dominance, where patches of each Gleason grade received the highest weighted activation from prototypes of the same class and consistently lower cross-class responses. This behavior matches the objective of the Stage-3 attention training and the classwise discriminative loss, which suppressed cross-class prototype activation while promoting class-specific prototype usage. The remaining non-zero cross-class responses are clinically plausible due to the partial morphological overlap between neighbouring grades. The analysis reaffirmed the quantitative gains observed in the ablation studies.

\begin{table}[!ht]
\centering
\caption{\centering Grade-wise performance of ADAPT on test data}
\begin{tabular}{|l|ll|ll|}
\hline
\multirow{2}{*}{\textbf{Grade}} & \multicolumn{2}{c|}{\textbf{PANDA}}                            & \multicolumn{2}{c|}{\textbf{SICAP}}                            \\ \cline{2-5} 
                                & \multicolumn{1}{l|}{\textbf{F1 score}} & \textbf{Hamming loss} & \multicolumn{1}{l|}{\textbf{F1 score}} & \textbf{Hamming loss} \\ \hline
GG3                             & \multicolumn{1}{l|}{0.8483}            & 0.2341                & \multicolumn{1}{l|}{0.8101}            & 0.1764                \\ \hline
GG4                             & \multicolumn{1}{l|}{0.8912}            & 0.1699                & \multicolumn{1}{l|}{0.8762}            & 0.2100                \\ \hline
GG5                             & \multicolumn{1}{l|}{0.7542}            & 0.1238                & \multicolumn{1}{l|}{0.7764}            & 0.1848                \\ \hline
\end{tabular}
\label{test_both}
\end{table}
\vspace{-4pt}

The generalizability of our ADAPT framework was evaluated in terms of the class-wise performance on the PANDA and the independent SICAP test set, as reported in Table \ref{test_both}. Despite differences in the characteristics of the two datasets and their acquisition protocols, the model maintained competitive F1 scores across all Gleason grades on SICAP, with only moderate performance variation. This indicates that the learned prototypes consistently captured transferable, class-relevant morphological patterns instead of the dataset-specific attributes. This study corroborates the robustness and cross-dataset applicability of ADAPT.

\vspace{-4pt}
\begin{figure}[!ht]
    \centering
    \includegraphics[width=0.25\textwidth, trim=0.2cm 0.2cm 0.2cm 0.2cm, clip]{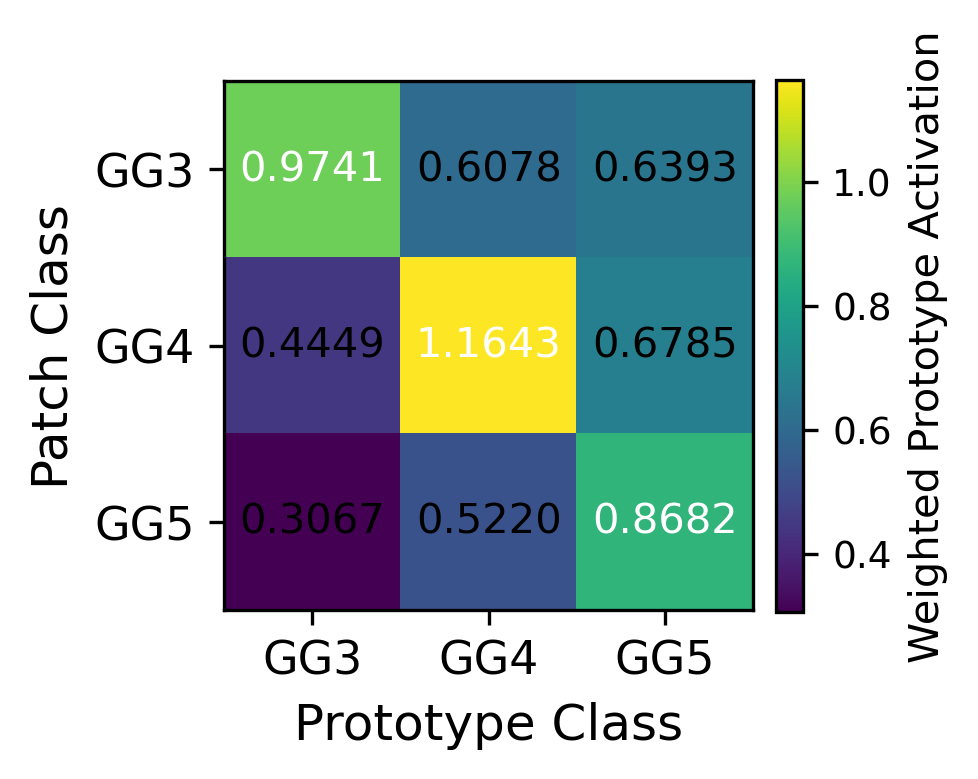}
    \caption{Prototype cross-activation matrix showing mean weighted prototype activations across patch classes for the PANDA dataset}
    \label{fig:prot_matrix}
\end{figure}
\vspace{-4pt}

\subsection{Qualitative analysis}

Fig. \ref{fig:prot_global} shows the learned prototypes for Gleason grades 3, 4, and 5, along with their attention weights and the corresponding annotated training patches, for the PANDA dataset. It is observed that prototypes having high attention scores ($> 0.5$) consistently localize well-defined glandular structures (characteristic of their assigned Gleason patterns), demonstrating clear agreement between the highlighted regions and the ground-truth masks. In contrast, prototypes with lower attention weights (particularly $< 0.5$) are visibly non-discriminative, often activating on regions such as stroma [4th row in (a) representing GP 3], benign epithelium [4th row in (b) representing GP 4], noisy regions [3rd row in (c) representing GP 5] or even incorrect Gleason morphologies [4th row in (c) representing GP 5]. This demonstrates that the attention module effectively down-weights the prototypes capturing background or misleading patterns, while preserving those that encode class-specific cancer morphology. This helps the model to base its decisions on clinically meaningful prototypes; thus, improving the transparency and reliability of its reasoning.

\begin{figure*}[!ht]
    \centering
    \includegraphics[width=0.85\linewidth]{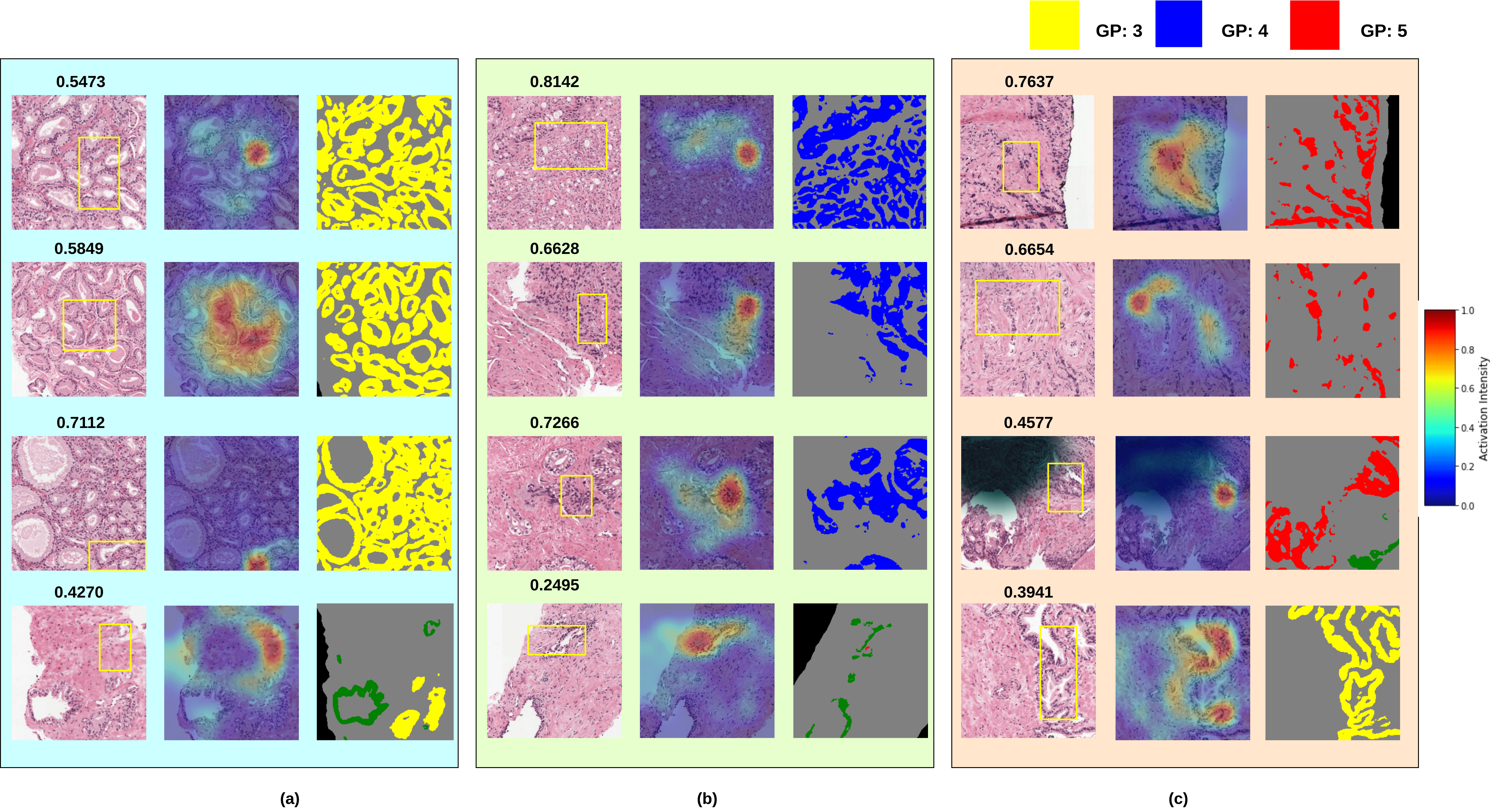}
    \caption{Prototype visualization as training patches for Gleason grades (a) 3, (b) 4, and (c) 5, respectively, for the PANDA dataset. Within each group, Column 1 shows the original training patch with the prototype-activated region highlighted using a bounding box; Column 2 displays the corresponding prototype self-activation heatmap; and Column 3 represents the ground-truth Gleason mask}
    \label{fig:prot_global}
\end{figure*}

 Fig. \ref{fig:test_images} assesses the interpretability of the ADAPT framework, in terms of prototype-based explanation for two representative WSIs from the PANDA and SICAP datasets, respectively, having complex heterogeneous Gleason compositions [GS 7 (PANDA): 4+3 and GS 9 (SICAP): 4+5]. For each test image, we visualize how the model arrived at its WSI-level Gleason prediction.  Patches having highest predicted probabilities are identified for each Gleason pattern present at the WSI-level, and their correspondence shown to the most similar learned prototypes. It is demonstrated (with heatmaps and bounding boxes) that the morphological structures corresponding to the top-scoring patch for each GP, and its most similar prototype, drive the prediction. The visualization consistently demonstrates that GP-specific prototypes activate on the appropriate glandular formations; such as, well-formed glands for GP 3, fused and poorly formed glands for GP 4, or solid sheets for GP 5. Similarity scores, prototype activations, and ground-truth masks collectively demonstrate that the slide-level decisions of ADAPT are grounded in localized, clinically meaningful morphological patterns. Further, the consistent performance across both PANDA and SICAP demonstrates that the learned prototype reasoning generalizes reliably.

\begin{figure*}[!ht]
    \centering
    \includegraphics[width=0.80\textwidth]{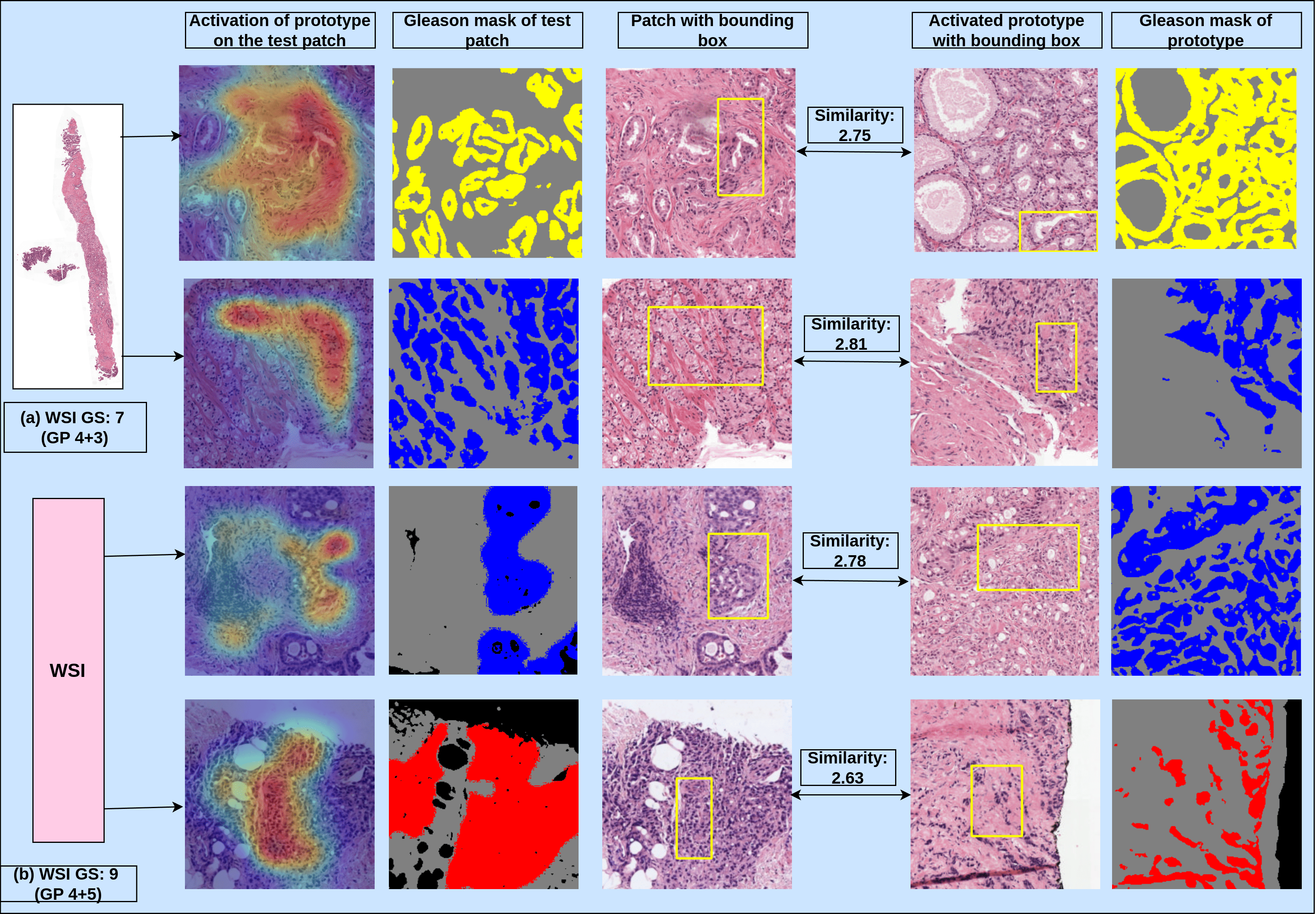}
    \caption{Prototype-based explanations for test WSIs from (a) PANDA, (b) SICAP datasets (WSI-level images are not available for SICAP). For each  Gleason pattern present at WSI-level, the most confident patch predicted by ADAPT is visualized in terms of nearest prototype activation, bounding-box correspondence, Gleason mask, and similarity to the nearest prototype}
    \label{fig:test_images}
\end{figure*}

\section{Conclusion}
\label{sec:conc}

A new prototype-based framework was developed to improve interpretability in prostate cancer grading from whole-slide histopathology images. Stage 1 patch-level pretraining enabled the learning of robust prototypical features associated with each Gleason grade; thereby, providing a strong initialization for WSI-level grading in a weakly-supervised MIL setting. The Stage 1 model was then innovatively fine-tuned for WSI-level grading in an MIL framework. Thus, Stage 2 addressed the domain shift between patch and slide distributions by introducing prototype-aware regularizers, such as positive alignment loss and negative repulsion loss. While the positive alignment loss encouraged recovery of missed evidence for true classes, the negative repulsion loss aimed to suppress misleading evidence responsible for false positives. Finally, the novel attention-based dynamic pruning mechanism in Stage 3 selectively emphasized more informative prototypes for generating optimal performance.

Experimental results demonstrated the contribution of each stage to the grading performance, with consistent gains observed across different prototype configurations. Qualitative analysis further revealed that the prototypes receiving higher weights corresponded to diagnostically relevant regions; while the lower weight prototypes were largely non-discriminative -- activating on regions such as stroma, benign epithelium or noise. The results reaffirmed the importance of the attention-based dynamic pruning mechanism following Stages 1 and 2. The consistency of the prototype–patch correspondence, over both the PANDA and SICAP datasets, highlights the robustness and generalizability of the ADAPT framework to varying data distributions and acquisition settings. Overall, the approach provides a reliable and interpretable foundation for assisting pathologists in the grading of prostate cancer.

Future study will investigate other strategies for simplifying the learned prototype set, based on their relevance. Extending the evaluation to larger, multi-institutional cohorts will help assess robustness under broader criteria of data heterogeneity. Incorporating complementary clinical or contextual information while adapting the framework to other histopathological grading tasks also constitute promising directions of research.

\bibliographystyle{unsrt}  
\bibliography{references}

\end{document}